\documentclass[10pt,twocolumn,letterpaper]{article}
\usepackage{times}
\usepackage{epsfig}
\usepackage{graphicx}
\usepackage{amsmath}
\usepackage{amssymb}
\usepackage{subcaption}
\usepackage{amsthm}
\date{}

\usepackage[breaklinks=true,bookmarks=false]{hyperref}

\usepackage{xifthen}
\usepackage{multirow}
\newtheorem{example}{Example}
\newtheorem{definition}{Definition}

\newtheorem{theorem}{Theorem}

\DeclareMathOperator{\mean}{AM}
\DeclareMathOperator{\doff}{doff}
\DeclareMathOperator{\dof}{dof}
\DeclareMathOperator{\argmax}{Argmax}
\DeclareMathOperator{\cnn}{\mathsf{\it{CNN}}}
\DeclareMathOperator{\conv}{Conv}

\DeclareMathOperator{\relu}{ReLU}
\DeclareMathOperator{\maxpool}{Pool}
\DeclareMathOperator{\ws}{WS}

\newcommand{\pmax}{p_{\it{max}}}
\newcommand{\bias}[2][]{\overline b^{(#2)}_{#1}}
\newcommand{\weight}[2][]{W^{(#2)}_{#1}}
\newcommand{\f}[1]{f^{(#1)}}
\newcommand{\domain}[1]{\mathcal D^{(#1)}}
\newcommand{\s}[2][]{s^{(#2)}_{#1}}
\newcommand{\val}[2][]{x^{(#2)}_{#1}}
\newcommand{\vval}[2][]{\overline x^{(#2)}_{#1}}
\newcommand{\convf}[2][]{\conv^{(#2)}_{#1}}
\newcommand{\convff}[2][]{\conv^{(#2),#1}}
\newcommand{\convfff}[3]{\conv^{(#2),#3}_{#1}}
\newcommand{\maxpoolf}[1]{\maxpool^{(#1)}}
\newcommand{\kernel}[2]{K^{(#1)}_{#2}}
\newcommand{\kernell}[2]{K^{(#1),#2}}

\newcommand{\pdelta}[2][]{\delta^{(#2)}_{#1}}
\newcommand{\tr}[2][]{\mathfrak T^{(#2)}_{#1}}
\newcommand{\con}[1]{\mathcal C^{(#1)}}

\newcommand{\im}{\overline{im}}

\newcommand{\set}[1]{\left \{ #1  \right \}}

\begin{document}
\title{Formal Verification of CNN-based Perception Systems}
\author{Panagiotis Kouvaros {\small and } Alessio Lomuscio\\
Department of Computing, Imperial College London, UK\\
{\tt\small \{p.kouvaros, a.lomuscio\}@imperial.ac.uk}}
\maketitle
\begin{abstract}
We address the problem of verifying neural-based perception systems
implemented by convolutional neural networks. We define a notion of local
robustness based on affine and photometric transformations. We show the
notion cannot be captured by previously employed notions of robustness. The
method proposed is based on reachability analysis for feed-forward neural
networks and relies on MILP encodings of both the CNNs and transformations
under question. We present an implementation and discuss the experimental
results obtained for a CNN trained from the MNIST data set.

\end{abstract}

\section{Introduction}\label{sec:intro}
Autonomous systems are forecasted to revolutionise key aspects of modern
life including mobility, logistics, and beyond. While considerable
progress has been made on the underlying technology, severe concerns
remain about the safety of the autonomous systems under
development. 

One of the difficulties with forthcoming autonomous systems is that
they incorporate complex components that are not programmed by engineers,
but are synthesised from data via machine learning methods (often
called \emph{learning-enabled components}).  Additionally, there is an
increasing trend to deploy autonomous systems in safety-critical areas
including autonomous vehicles. These two aspects taken together call for
the development of rigorous methods for the formal verification of
autonomous systems based on learning-enabled components, see,
e.g., DARPA's Assured Autonomy
program~\footnote{\url{https://www.darpa.mil/program/assured-autonomy}.}.

One of the most widely considered architectures combining
traditionally programmed and learning-enabled components is control
systems (or decision making and planning modules) paired with
perception systems based on data-trained classifiers. Forthcoming
autonomous vehicles, among other systems, use this architecture. While
there are methods to formally assess the adequacy of decision making
and control components, no existing technique can provide formal
guarantees about the correctness of perception systems based on
machine learning methods. Not only this makes it hard to formally  establish
 the correctness of perception systems before deployment, it
also impedes the provision of formal assurances on the overall
autonomous system, since the control components
normally receive input from the classifiers of the perception
component.

In this paper we aim to make a contribution towards solving this
problem. Specifically, we introduce a method to give formal guarantees
to perception systems based on convolutional neural networks
(CNNs)~\cite{GoodfellowBengioCourville16}. Towards this end, we
develop a concept of robustness for CNN-based perception systems with
respect to affine and photometric transformations, i.e., changes in
the position and colour patterns of the image. These can be used to
anticipate the natural variability of the input to the perception
mechanism when the system is deployed. Intuitively, establishing that
for all the possible transformations of the input (up to a certain
measure) the CNN classifies the image in the same way helps us
understand the robustness of the perception system with respect to
said transformations.
Moreover, whenever misclassifications are found, these can usefully be
used to retrain the network and improve its accuracy.  The method we
introduce relies on recasting the problem in mixed-integer linear
programming (MILP) and leverages on recent work on reachability
analysis for feed-forward neural networks, which we extend here to
CNNs. Two key features of our approach which make it different from
recent approaches discussed below are that: (i) we study robustness
with respect to image transformations as opposed to pixel alterations,
(2) the method is complete, i.e., if a missclassification exists with
respect to a transformation the method is theoretically guaranteed to
find it. Completeness is particularly important for the purposes of
guaranteeing correctness.


The rest of the paper is organised as follows. After discussing
related work below, in Section~\ref{sec:preliminaries} we introduce
CNNs and image transformations. In Section~\ref{sec:robustness} we
develop the concept of local transformational robustness and give
examples of applications.  We turn to the verification problem for
CNN-based classifiers in Section~\ref{sec:verification} where we
construct the MILP encoding of the problem and show the soundness and
completeness of the approach. We present an implementation in
Section~\ref{sec:experiments} and report experimental results obtained
for a CNN trained from the MNIST dataset. We conclude in
Section~\ref{sec:conclusions}.

{\bf Related Work.} The notions of affine and photometric transformations
are well understood in the
literature~\cite{HartleyZisserman03,SonkaHlavacBoyle14}. Classifiers are
normally evaluated in statical terms against false positives and false
negatives~\cite{Bishop06}. While these studies are relevant, they do not
have the same aim as ours, which is to formally assess the correctness of a
classifier with respect to image transformations.

More relevant to our aims is the recent work~\cite{Gehr+18} on the
verification of feed-forward networks, which includes a section on the
local robustness of CNN-based classifiers. This works differs from
ours in several respects. Firstly, local robustness, as defined there,
captures regions (or abstract domains) of image perturbations.
Differently, we here deal with a wide range of geometric and
photometric transformations.  If we were to express these by means of
abstract domains, we would have to over-approximate the input domain,
thereby obtaining robustness violations not relevant to the
transformations under analysis. Instead, by developing the concept of
local transformational robustness we can provide tighter notions of
correctness than by purely working on local robustness.
The technical details of the two approaches are also
different. While~\cite{Gehr+18} uses abstract interpretation as
underlying method we here compile into MILP problems. While abstract
interpretation provides better performance than MILP, it comes at the
cost of completeness, which is a key requirement for the purposes of
verification.

Related to this, there is a fast increasing body of research addressing
reachability and robustness analysis for neural
networks~\cite{Katz+17,Ehlers17,Narodytska+17,LomuscioMaganti17,Huang+17}.
While our approach is similar in terms of the underlying method, the aims
are different since, rather than analysing reachability  and robustness in
abstract terms, we here intend to contribute to the concrete problem of
verifying learning-enabled perception modules. When perception systems
are considered in these works, they are assessed with respect to
pixel variations of the images and not transformations as we do here. 
We are not the first to suggest that image transformations should be
considered when assessing the robustness of image classifiers. This
point was raised very recently also by~\cite{Seshia+18} (referred to
there as ``semantic invariance''). However, that work is intended as
initial foundational work and proposes no method to solve the problem.

Related to this submission is also the unpublished~\cite{Pei+17} which
proposes a generic framework for evaluating the safety of computer
vision systems. In common with this paper it aims to establish
correctness of a classification with respect to image
transformations. Differently from ours, their approach however is
based on heuristic search for counterexamples. While it is shown to
provide advantages over previous state-of-the-art based on
gradient-based methods, it is necessarily incomplete, i.e., the method
cannot give guarantees that if a missclassifications exists this will
be found. In contrast we here focus on completeness as a requirement,
i.e., we guarantee that (up to computational power) \emph{any and all}
missclassifications will be found by the method. We cannot compare
experimental results to~\cite{Pei+17} as the paper is unpublished and
we do not have access to the tool described therein.

Lastly, as it will be clear from the rest of the paper, by using the
results here presented, whenever we identify that a perception module is
not robust, i.e., it classifies images differently following small
transformations, our procedures return the witness of this
misclassification. This can be seen as an adversarial example, see,
e.g.~\cite{Papernot+16}, for the relevant classifier. Research in
adversarial examples  directly aims at finding these counterexamples.
However, none of these methods is complete, hence they cannot be used to
show formal correctness which is the primary aim of this work.

\section{Preliminaries}\label{sec:preliminaries}
Learning-enabled components for perception are typically image
classifiers implemented via CNNs. We here summarise CNNs, CNN-defined
image classifiers,  and define image transformations as distortions.
We refer to~\cite{GoodfellowBengioCourville16,Rushakovsky+15} for more details.


{\em CNNs} are directed acyclic graphs whose nodes are structured in
  layers.  The first layer is said to be the {\em input layer}, the
  last layer is referred to as the {\em output layer}, and every layer
  in-between is called a {\em hidden layer}.  Every layer is either a
  {\em fully connected layer} or a {\em convolutional layer}.

Each node of a fully connected layer is  connected to every node in the
preceding layer, whereas every node of a convolutional layer is only
connected to a {\em rectangular neighbourhood} of nodes in the preceding
layer. In either case every connection is  associated with a weight.  

A fully connected layer is composed of two phases, whereas a convolutional
layer comprises three phases. In both cases the
first phase linearly activates the nodes by outputting the weighted sum of
their inputs, and the second phase applies a non-linear {\em activation
function} to the linear activation. Here we consider the Rectified Linear
Unit (ReLU) whose output is  the maximum between~0 and the linear
activation. The third phase of a convolutional layer further
applies a {\em pooling function} to collapse rectangular neighborhoods of
activations into {\em representative activations}.  Here we focus on the
max-pooling function where the maximum activation is chosen as the
representative one.


CNNs are routinely used to solve image classification problems. The
problem concerns the approximation of an unspecified function
$f^\star \colon
\mathbb R^{\alpha \times \beta \times \gamma} \rightarrow
\set{1,\ldots,c}
$
that takes as input an image $\im$ from an unknown distribution
$\mathbb R^{\alpha \times \beta \times \gamma}$ (where $\alpha \times
\beta$ are the pixels of the image and $\gamma$ is the number of the
color bands, e.g.  RGB) and determines among a set of classes
$\set{1,\ldots,c}$ the class to which $\im$ is a member. The problem
is tackled by training a CNN by means of a labelled training set
thereby setting the weights of the network so that its output
approximates~$f^\star$~\cite{GoodfellowBengioCourville16}.  Following
this, the CNN can be used to infer the class of novel images by
feeding an image to the input layer and then propagating it through
the network. At each step the outputs of the nodes in a layer are
computed from the outputs of the previous layer, resulting in the
input to the output layer. The output layer is a fully connected layer
that composes precisely a node per label class and whose activation
function outputs the index of the node with the largest linear
activation as the classification decision.  In the following we assume
the weights of a CNN have been trained.

To fix the notation, we use $[n]$ for $\set{1,\ldots,n}$ and
$[n_1,n_2]$ for $\set{n_1,\ldots,n_2}$.  We fix a CNN $\cnn$ with a
set of layers $[n]$. We assume the nodes in a convolutional layer
are arranged into a three dimensional array.  Interchangeably we
sometimes treat this arrangement as reshaped into a vector. The nodes
in a fully connected layer are arranged into a vector.  The output of
the $(j,k,r)$-st ($j$-th, respectively) node in a convolutional (fully
connected, respectively) layer~$i$ is represented by $\val[j,k,r]{i}$
($\val[j]{i}$, respectively). We denote by $\vval[]{i}$ the vector of
all the nodes' outputs in layer~$i$.  We write $\s i$ for the size of
layer~$i$ and $\s[j]{i}$ for the size of the~$j$-th dimension of a
convolutional layer~$i$.

We now proceed to formally define every layer $2 \leq i \leq n$ as a
function $\f i \colon \mathbb R^{\s[]{i-1}} \rightarrow \mathbb R^{\s i}$
and consequently the CNN composing these.  We begin with a fully connected
layer$~i$. The layer is associated with a weight matrix $\weight[]{i} \in
\mathbb R^{\s[]i \times \s[]{i-1}}$ and a bias vector $\bias[]{i} \in
\mathbb R^{\s[]{i}}$.  The linear activation of the layer is given by the
weighted sum $\ws(\vval[]{i-1}) = \weight[]{i} \cdot \vval[]{i-1} +
\bias[]{i}$. The function computed by the
layer can then be defined as 
$
\f i(\vval[]{i-1}) = \Box \left(\ws\left(\vval[]{i-1}\right)\right)
$, 
where $\Box \in \set{\relu,\argmax}$ with the function $\relu(x) =
\max(0,x)$ being applied element-wise to the linear activation.  We now
give the definition of a convolutional layer~$i$. The layer is associated
with a group $\convff[1]{i},\ldots,\convff[k]{i}$ of $k \geq 1$ convolutions
and a max-pooling function $\maxpoolf i$.  Each convolution
$
\convff[j]{i} \colon \mathbb R^{\s[]{i-1}} \rightarrow
\mathbb{R}^{(\s[1]{i-1} - p +1) \times (\s[2]{i-1}-q+1)}
$
is parameterised over a weight matrix (kernel) $\kernell{i}{j} \in \mathbb
R^{p \times q \times \s[3]{i-1}}$, where $p \leq \s[1]{i-1}$, $q \leq
\s[2]{i-1}$, and a bias vector $\bias[j]{i}$. The $(u,v)$-st output of the
$j$-st convolution  is
given by
$
\convfff{u,v}{i}{j}(\vval[]{i-1}) = \kernell{i}{j} \cdot
\vval[{[u,u'],[v,v'],[\s[3]{i-1}]}]{i-1} + \bias[j]{i}, 
$
where $u' = u+p-1$ and $v' = v+q-1$. 
Given the outputs of each of the convolutions, the linear activation of the layer
$
\convf[]{i} \colon \mathbb R^{\s[]{i-1}} \rightarrow
\mathbb{R}^{(\s[1]{i-1} - p +1) \times (\s[2]{i-1}-q+1) \times k}
$
forms a three-dimensional matrix, i.e. $\convf[]{i} =
\begin{bmatrix} \convf[1]{i} & \ldots & \convf[k]{i} \end{bmatrix}$.
The non-linear activation of the layer is then computed by the application
of the $\relu$ function. In the final phase, the max-pooling
function collapses  neighborhoods of size $p' \times q'$ of the
latter activations to their maximum values. Formally,  $\maxpoolf{i} \colon
\mathbb R^{u \times v \times r} \rightarrow \mathbb R^{\frac{u}{p'} \times
\frac{v}{q'} \times r}$, where $u = (\s[1]{i-1} - p +1)$ and $v =
(\s[2]{i-1}-q+1)$, is defined as follows:
\(
\maxpoolf{i}_{u,v,r} = \max \left( \convf[{[(u-1) p +1,u\cdot
p],[(v-1)
q+1,v\cdot q],r}]{i} \right)
\)
The function computed by the layer is then defined by
\(
\f i(\vval[]{i-1}) = \maxpoolf i
\left(\relu \left(\convf[]{i}\left(\vval[]{i-1}\right)\right)\right).
\)

Given the above a convolutional neural network can be defined as the
composition of fully connected and convolutional layers.

\begin{definition}[Convolutional neural network]
	A convolutional neural network $\cnn : \mathbb R^{\alpha \times \beta
	\times \gamma} 	\rightarrow [c]$ is defined as 
	\[\cnn(\im) = \f{n}(\f{n-1}(\ldots \f{1}(\overline{x})\ldots))\]
	where  $\im \in \mathbb R^{\alpha \times \beta \times
	\gamma}$, $\f{1},\ldots,\f{n-1}$ are either $\relu$ fully connected or
	convolutional layers,  $\f{n}$ is an $\argmax$ 	fully connected
	layer, and $[c]$ is a set of class labels.
\end{definition}



{\bf Image transformations.} We briefly describe affine and photometric 
transformations.  An affine transformation is a mapping between affine
spaces that preserves collinearity and ratios of
distances~\cite{HartleyZisserman03}. It is represented in vector algebra as
follows.
\[
\begin{pmatrix}
	x' \\ y' \\ 1
\end{pmatrix}
=
\begin{bmatrix}
	a_{11} & a_{12} & t_x \\
	a_{21} & a_{22} & t_y \\
	0 & 0 & 1
\end{bmatrix}
\begin{pmatrix}
	x  \\ y \\ 1 
\end{pmatrix}
\]
where $(x,y)$ is a point, $A = \begin{bmatrix} a_{11} & a_{12} \\ a_{21} &
	a_{22} \end{bmatrix}$ is a non-singular matrix, and $\overline t =
(t_x, t_y)^T$ is the {\em translation
		vector}.

An affine transformation whereby $A$ equals the identity matrix is said to be a {\em translation}.
An affine transformation is referred to as {\em scaling} if $A = \sigma I$,
$t_x=0$ and $t_y=0$ for a scale factor $\sigma$.   In the case where
$\sigma < 1$ the scaling is called {\em subsampling},  whereas in the case
where $\sigma > 1$ the scaling is known as {\em zooming}. Given a
transformation $\mathfrak T$ we write $\doff(\mathfrak T)$ for the tuple of
degrees of freedom in $\mathfrak{T}$, e.g., if $\mathfrak T$ is a
translation, then $\doff(\mathfrak T) = (t_x,t_y)$. Given $d \in \mathbb
R^{|\doff(\mathfrak T)|}$, we denote by $\mathfrak T[d]$ the {\em
concretisation} of $\mathfrak T$ whereby every parameter $\doff(\mathfrak
T)_i$ is set to $d_i$.  In this paper we consider translations and scaling
transformations, both in isolation and in composition with themselves. 

A photometric transformation  is an affine change in the intensity of the
pixels~\cite{SonkaHlavacBoyle14}. It is defined as  $f(p) = \alpha p +
\beta$.  If  $0 < \alpha < 1$, then  the transformation reduces the
contrast of the image. Otherwise, if $\alpha > 1$, the transformation
increases the contrast of the image. The factor $\beta$ controls the
luminosity of the image with higher values pertaining to brighter images.

\begin{figure*}
    \centering
    \begin{subfigure}[b]{0.15\textwidth}
		\includegraphics[width=\textwidth]{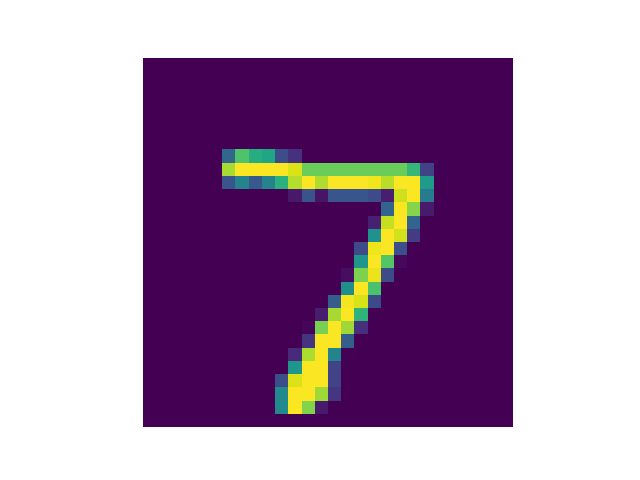}
        \caption{Original.}
        \label{fig:mnist-original}
    \end{subfigure}%
    \begin{subfigure}[b]{0.15\textwidth}
		\includegraphics[width=\textwidth]{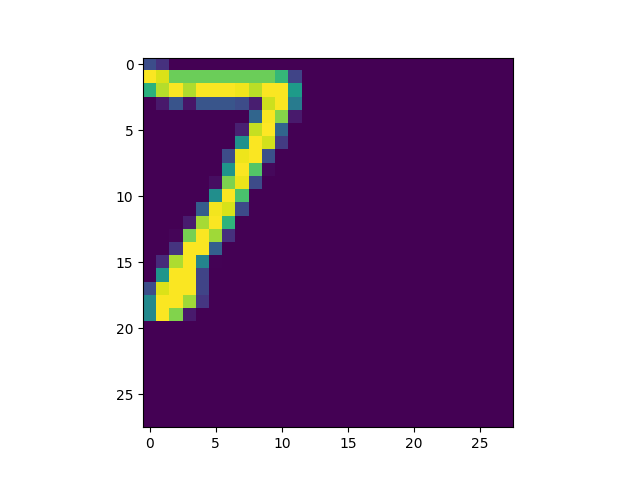}
        \caption{Translation.}
        \label{fig:mnist-translation}
    \end{subfigure}%
    \begin{subfigure}[b]{0.15\textwidth}
		\includegraphics[width=\textwidth]{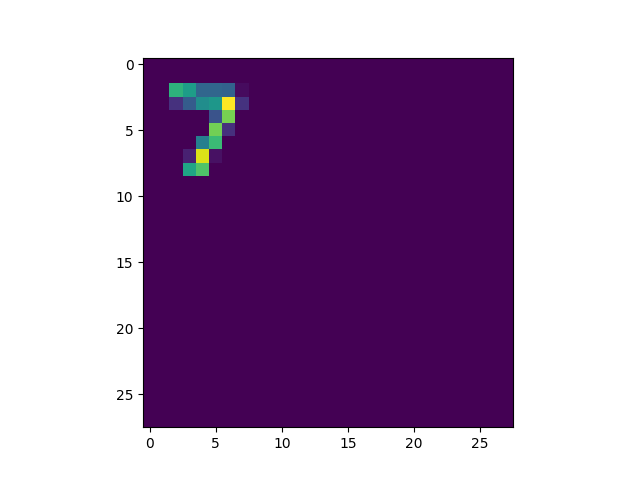}
		\caption{Subsampling.}
        \label{fig:mnist-subsample}
    \end{subfigure}%
    \begin{subfigure}[b]{0.15\textwidth}
		\includegraphics[width=\textwidth]{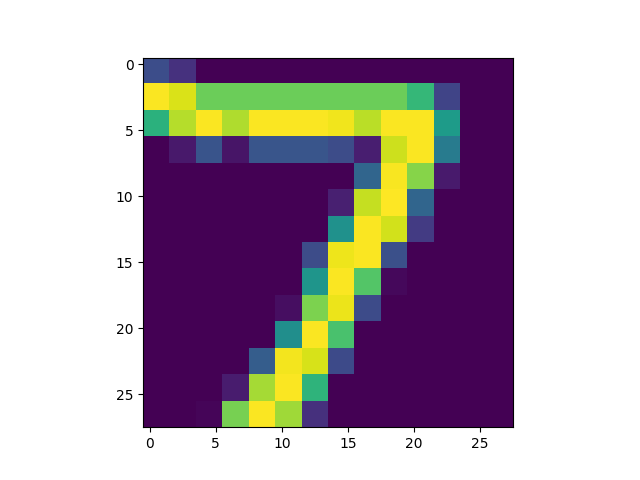}
		\caption{Zooming.}
        \label{fig:mnist-zoom}
    \end{subfigure}%
    \begin{subfigure}[b]{0.15\textwidth}
		\includegraphics[width=\textwidth]{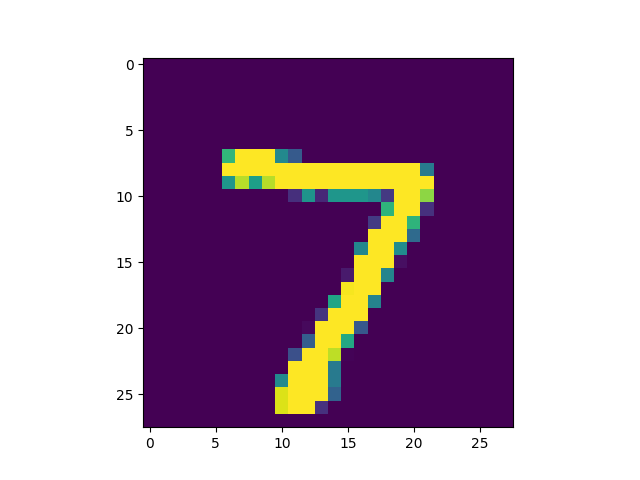}
		\caption{Luminosity.}
		\label{fig:mnist-brightness}
    \end{subfigure}%
    \begin{subfigure}[b]{0.15\textwidth}
		\includegraphics[width=\textwidth]{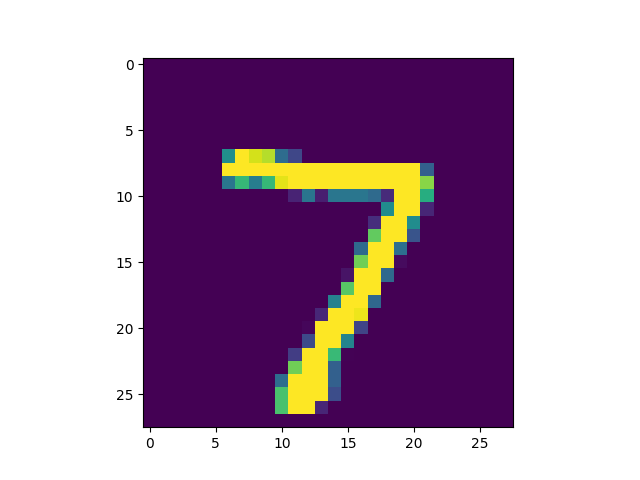}
		\caption{Contrast.}
		\label{fig:mnist-contrast}
    \end{subfigure}%
    \begin{subfigure}[b]{0.15\textwidth}
		\includegraphics[width=\textwidth]{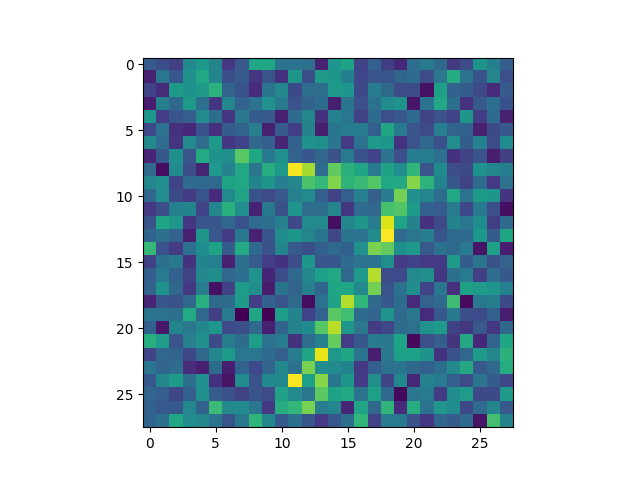}
		\caption{Random noise.}
		\label{fig:mnist-noise}
    \end{subfigure}%
	\caption{Affine and photometric transformations applied on an image
	from the MNIST dataset.} 
	\label{fig:transforms}
\end{figure*}

\section{Local transformational robustness}\label{sec:robustness}
CNNs are known to be susceptible to adversarial attacks. Slight
perturbations of an image, even imperceptible to the human eye, can
cause the misclassification of the
image~\cite{Papernot+16}. As mentioned in the introduction a
recent body of work aims at establishing the {\em adversarial local
robustness} of neural networks against said
attacks by adopting a
verification perspective~\cite{Narodytska+17,Gehr+18,Huang+17,Katz+17}.
Specifically, given an image $\im$, current
state-of-the-art methods check whether all images $\im'$ with
$\lVert \im' - \im \rVert_p \leq \delta$, for some $L_p$-norm, are
classified as belonging to the same class with $\im$.
While this approach can effectively identify probable
misclassification when images marginally differ, e.g., by a few
different pixels, we argue that it is not adequate to identify a much
wider class of variations to an image that should not
affect its classification. 

Intuitively, a perception system should be assessed as robust whenever
it assigns the same label to every pair of images that human
perception finds indistinguishable.  For instance, autonomous vehicles
are expected to correctly identify objects irrespective of the angle
and position of the camera, the environment light conditions, the
position of the object in the image, the distance of the object from
the camera, etc. It follows that the \emph{classification decision
should be the same irrespective of minor translations,
 scaling, and photometric transforms of the image from the camera}
(Figure~\ref{fig:transforms}).

Crucially, adversarial local robustness depends on the encapsulation
of all indistinguishable images within the space induced by bounding
some distance metric, typically the infinity norm.  Consequently, in
order to capture said transformations in present approaches, the
bound on the distance between the images has to be so large that,
intuitively, distinguishable images would have to be classified in the
same way to meet the predefined distance.

For example, consider Figure~\ref{fig:transforms}. The $L_\infty$-norm
between the original image and the brightened image is~17.16, whereas
the distance between the original image and the image composed with
random noise is~16.02. Clearly the brightened image should be
classified as the original one whereas the classification of the
noisy one is not as clear. Following these observations, we propose
the formal verification of the robustness of CNNs not just up to a
distance metric, but up to transformations.  We call this the {\em
transformational robustness property}.

\begin{definition}[Local transformational robustness]
	A convolutional neural network $\cnn$ is said to be 
	{\em locally transformationally robust}
	for an image $\im$ and a transformation $\mathfrak T$ with
	domain~$\mathcal D \subseteq \mathbb R^{|\doff(\mathfrak T)|}$
	if for all $d \in \mathcal D$ we have that
	$\cnn\left(\mathfrak T[d](\im)\right) = \cnn(\im)$.
\end{definition}

In other words, we say that $\cnn$ is locally transformationally robust if
the image in question is classified in the same way  following the
transformation under analysis. The emphasis on ``local'' in the definition
above follows~\cite{Katz+17,Narodytska+17} and refers to the fact that the
property is relative to a single image. In contrast, global
transformational robustness is defined independently of the image. We do
not add this definition as the method provided below only deals with local
properties. Also note that the definition above is parameterised over a
domain for the values of the parameters. Indeed, for verification purposes,
it may be preferable to check only a range of values, e.g., in photometric
transformations it is typically not relevant to consider the transformation
for which every pixel is mapped to the black colour. Hereafter, we assume
that the transformation domain is a linearly definable set.

\section{Verifying local transformational robusteness}\label{sec:verification}

In this section we develop a procedure for verifying local
transformational robustness by recasting the problem into MILP. The
procedure takes as input a trained CNN $\cnn$, an image $\im$ with
class label $l_{\im}$, and a sequence of 
$k$ transformations $\tr[]{1},\ldots,\tr{k}$  with domains
$\domain 1,
\ldots,\domain k$, respectively. It then constructs an {\em
transformational} CNN $\overline{\cnn}$ by treating the
transformations as additional layers and appending them to the input
layer.  More precisely, $\overline{\cnn} =
\f{n+k+2}(\ldots \f{1}()\ldots)$, where $\f{n+k+2},\ldots,\f{k+2}$ are
the original layers of $\cnn$, $\f{k+1}$ is a {\em perturbation layer} 
and $\f{k},\ldots,\f{1} = \tr[]{k},\ldots,\tr[]{1}$. The perturbation layer
helps to simulate subpixel transformations; this will be discussed below.
The transformational CNN is then encoded into
mixed-integer linear constraints
layer by layer. For each layer and transformation~$i$ we now give the
set of linear constraints $\con i$ characterising the various layers and 
transformations that we consider.

{\bf Photometric transformations.}  We begin by presenting the
mixed-integer linear encodings for photometric transformations. A
photometric transformation has two degrees of freedom: the factor
$\mu$ that handles the contrast of the image, and the factor $\nu$
which controls the luminosity of the image.
We assume the minimum value of a pixel is~0 and write $\pmax$ for its
maximum value.  The possible instantiations of a photometric
transformation are expressed by the following linear constraints.
\begin{align}
\lambda_d \geq \min_d(\domain i) \text{ and } \lambda_d \leq \max_d(\domain i),
d \in \dof(\tr[]{i}) \tag{$C_1$}
\end{align}
where each $\lambda_d$ is an LP variable controlling the values of the
factor $d$, and $\min_d(\domain i)$ ($\max_d (\domain i)$, respectively)
denotes the minimum (maximum, respectively) value for the factor~$d$ in
$\domain i$.  To encode the result of a photometric transformation on a
pixel we simply express the corresponding linear operation. 
\begin{align}
	&\tr[px]{i} (\vval[]{i-1}) = \mu \cdot \vval[px]{i-1}  + \nu \tag{$C_2$}
\end{align}

Given this, the linear encoding of a photometric
transformation~$\tr[]{i}$ is then given by the set of constraints
$\con{i} = C_1 \cup C_2$.

{\bf Affine transformations.} We now turn to translating affine
transformations. For each case we use a set of constraints of the form
$C_1$ capturing the set of possible instantiations of the
transformation. Also, for each case, we introduce a binary variable
$\pdelta[d]{i}$ for every instantiation~$d$. The variable represents
whether the corresponding instantiation is the one being applied. At
any time we assume the application of precisely one instantiation by
imposing the constraint:
\begin{align*}
	\sum_{d \in \domain i} \pdelta[d]{i} = 1 \tag{$C_3$} 
\end{align*}
We also force a bijection between the set of $\delta$ variables and
the instantiations they represent by assuming: 
\begin{align*}
\sum_{d \in \domain i}  d_j \cdot \pdelta[d]{i}  =
\lambda_j  
\text{ for each d.o.f. } j \text{ in } \tr[]{i}  \tag{$C_4$}
\end{align*}
So $\pdelta[d]{i} = 1$ iff, for each d.o.f. $j$, the LP variable
$\lambda_{j}$ representing~$j$ equals $d_j$.

We now present the encodings of each of the transformations. We begin
by considering
geometric
translation. A translation shifts the location of every pixel as per
the translation vector $(u',v')$:
\begin{align*}
	&\tr[u,v,r]{i}(\vval[]{i-1}) = \sum_{(u',v') \in \domain i}
	\vval[u+u',v+v',r]{i} \cdot 
\pdelta[u',v']{i} \tag{$C_{5}$}
\end{align*}
Therefore, the linear encoding for a translation~$\tr[]{i}$ is defined
as $\con{i} = C_1 \cup C_3 \cup C_4 \cup C_{5}$.

Next we consider subsampling. The transformation collapses
neighbourhoods of points to a point whose value is a statistic
approximation of the neighbourhood. Here we consider the arithmetic
mean value. The size of the neighbourhood is controlled by the scaling
factor~$d$.
\begin{align*}
	\tr[u,v,r]{i}(\vval[]{i-1}) = &\sum_{d \in \domain i}
	\mean(\vval[{[(u-1)d,u\cdot d],[(v-1)d,v\cdot d],r}]{i-1}) \cdot
	\pdelta[d]{i} \tag{$C_{6}$}
\end{align*}
It follows that the linear encoding for a subsampling~$\tr[]{i}$ is
defined as $\con{i} = C_1 \cup C_3 \cup C_4 \cup C_{6}$.

Lastly, we consider zooming. Zooming replicates the value of a pixel
to a rectangular neighbourhood of pixels. The value of the
neighbourhood is controlled by the scaling factor~$d$.
\begin{align*}
	\tr[u,v,r]{i}(\vval[]{i-1}) = \sum_{d \in \domain i} \vval[ \lceil
	\frac{u}{d} \rceil, \lceil \frac{v}{d} \rceil,r]{i-1} \cdot
	\pdelta[d]{i} \tag{$C_{7}$}
\end{align*}

Therefore, the linear encoding for a zooming~$\tr[]{i}$ is defined as
$\con{i} = C_1 \cup C_3 \cup C_4 \cup C_{7}$.


{\bf Perturbations.}  Next we describe the linear encoding of the
perturbation layer. The layer is introduced to simulate subpixel
transformations. Whereas interpolation can be recast to MILP so as  compute
the transformations, the number of binary variables needed for the encoding
is overwhelming w.r.t the performance of the verifier, as also showed by implementations
we developed. In contrast, the perturbation layer results in a
conciser program whereby only two inequalities per pixel are used. The
layer perturbs each pixel within an $L_\infty$  norm-ball. 
Given the radius of the norm-ball is sufficiently big, i.e., equal to the
variance of the image, the layer captures each  of the images 
 that could be computed via interpolations. Consequently it enables
the derivation of stronger robustness guarantees than the ones that could
be given by using a specific interpolation. The linear encoding of the
layer is given for each pixel $px$ by the following constraints.
\begin{align*}
	& \vval[px]{k+1}  - \vval[px]{k} \leq  \rho \tag{$C_{8}$}\\
	& \vval[px]{k+1}  - \vval[px]{k} \geq - \rho \tag{$C_{9}$}\\
\end{align*}
where $\rho$ is a given radius. 

We now consider the encodings of the layers of the CNN under question. We
separately translate these phase by phase.

{\bf Fully connected layer.}  We give the MILP description of the
weighted sum and $\relu$ functions. We do not directly represent the
$\argmax$ function since the function is not essential in expressing
the local
transformational
robustness properties that we are here concerned with (see below). The
weighted sum function is given be the following:
\begin{align*}
\ws_j(\vval[]{i-1}) = \weight[j]{i} \cdot  \vval[j]{i-1} + \bias[j]{i}
\tag{$C_{10}$}
\end{align*}
To capture the piecewise-linearity of the $\relu$ function we introduce a
binary variable $\pdelta[j]{i}$ per node~$j$ that represents whether the
output of the node is above~0~\cite{Akintunde+18}.
\begin{align*}
	&\relu_j(\vval[]{i-1}) \geq 0 \nonumber \tag{$C_{11}$} \\
	&\relu_j(\vval[]{i-1}) \geq \ws_j(\vval[]{i-1}) \tag{$C_{12}$} \\
	&\relu_j(\vval[]{i-1}) \leq \ws_j(\vval[]{i-1}) + M\pdelta[j]{i}
	\tag{$C_{13}$} \\
	&\relu_j(\vval[]{i-1}) \leq M(1 - \pdelta[j]{i}) \tag{$C_{14}$}
\end{align*}
Above $M$ is a sufficiently large number.  Therefore the linear encoding of
a fully connected layer~$i$ is defined by $\con{i} = C_{10} \cup \ldots
\cup C_{14}$.

{\bf Convolutional layer.} In addition to the $\relu$ phase, a
convolutional layer includes a convolution and a max-pooling phase.
Similarly to the weighted-sum function, a convolution is a linear
operation on the input of the layer and can be encoded by the following:
\begin{align*}
	\convfff{u,v}{i}{j}(\vval[]{i-1}) = \kernel{i}{j} \cdot
\vval[{[u,u'],[v,v'],[\s[3]{i-1}]}]{i-1} + \bias[j]{i}, \nonumber \\
u' = u+p-1, v' = v+q-1. \tag{$C_{15}$}
\end{align*}
A max-pooling function is parameterised over the size of the groups of
pixels over which the max-pooling is performed. Previous linear encodings
of the function use a binary variable per node in a
group~\cite{Dvijotham+18}. We here provide
an encoding that uses logarithmically less variables. Specifically, to select the maximum
value from a group we introduce a sequence of binary variables.
The sequence's binary number expresses the node in a group whose value is
maximum.  Since the size of the group is $p\cdot q$ we use $\lceil \log_2(p
\cdot q) \rceil$ variables. To facilitate  the presentation of the
corresponding linear constraints, we write $\mathbf n$ for the binary
representation of $n \in \mathbb Z$. We denote by $|\mathbf n|$ the number
of binary digits in $\mathbf n$. Given $j \in |\mathbf n|$, $\mathbf{n_j}$
expresses the $j$-th digit in $\mathbf n$ whereby the first digit is the
least significant bit.  If $j > |\mathbf n|$, then we assume that
$\mathbf{n_j}=0$. The linear representation of the max-pooling function for
a pixel $px = (px_{\alpha},px_{\beta},px_{\gamma})$ and pool size~$p\times
q$  is given by the
following.
\begin{align}
	&\maxpoolf{i}_{px} \geq  \convf[(px_{\alpha}-1)p +
	u',(px_{\beta}-1)q + v',px_{\gamma}]{i},
	\nonumber \\  & \hskip 4em u' \in [p], v' \in [v] \tag{$C_{16}$}\\
	&\maxpoolf{i}_{px} \leq  \convf[(px_{\alpha}-1)p + u',(px_{\beta}-1)q
	+ v',px_{\gamma}]{i} +
	\nonumber \\
	&\hskip 4em M\sum_{j \in [|\mathbf{p \cdot q}|]} \mathbf{z_j} +
	(1-2\mathbf{z_j})\pdelta[px,j]{i}, \nonumber \nonumber \\
	&\hskip 4em u' \in [p], v' \in [v],  z  = (u'-1)q + v' -1
	\tag{$C_{17}$}
\end{align}
where  $\pdelta[px,1]{i},\ldots,\pdelta[\lceil \log_2(px,p \cdot
q) \rceil]{i}$  are the binary variables associated with $px$. 
For the case where $p \cdot q$ is not a power of~2 we require from the
number represented by $\pdelta[px,1]{i},\ldots,\pdelta[\lceil \log_2(px,p \cdot
q) \rceil]{i}$ to lie within $0,\ldots,p \cdot q -1$. The constraint is
formally expressed as follows. 
\begin{align}
&\sum_{j \in |\mathbf{p \cdot q}|} \mathbf{z_j} + (1 - 2 \mathbf{z_j})
	\pdelta[px,q]{i} \geq 1, \nonumber \\ &\hskip 4em z \in [p \cdot
	q,2^{|\mathbf{p\cdot q}|}-1] \tag{$C_{18}$}
\end{align}
\begin{example} 
Assume the application of the max-pooling function on a group of nodes
$x_1,\ldots,x_4$. Let $y$ denote the output of the function. Its MILP
encoding as per~(16)-(19) uses two binary variables $\delta_1$, $\delta_2$
and is given by the following.
\begin{align*}
&	y \geq x_j, j \in [4] \\
&	y \leq x_1 + M(\delta_1 + \delta_2) \\
&	y \leq x_2 + M(1 - \delta_1 + \delta_2) \\
&	y \leq x_3 + M(\delta_1 + 1 - \delta_2) \\
&	y \leq x_4 + M(1 - \delta_1 + 1 - \delta_2)
\end{align*}
\end{example}

The linear encoding of a convolutional layer~$i$ is defined by $\con{i} =
C_{15} \cup \ldots \cup C_{18}$. 

Given the above, the linear encoding of a transformational CNN can be
defined as the union of the linear constraints characterising its
layers.


\begin{definition}[Linear encoding of transformational CNNs]
	The linear encoding of a transformational CNN
	$\overline{\cnn} = f^n(f^{n-1}(\ldots f^1()\ldots))$ is defined by
	$\con{\overline{\cnn}} = \con{1} \cup \ldots \cup \con{n}$.
\end{definition}

We can exploit the above encoding to construct a linear program whose
solution corresponds to a solution of the robustness problem for the
CNN under analysis. Indeed, we can extend the above encoding to
represent the local robustness property.

{\bf Local transformational robustness.} We now define linear
constraints to determine whether there are instantiations for the
transformations of the original image $\im$ whose output is classified
differently than $\im$. More specifically, we check whether there is a
linear activation in the output layer that is larger than the
activation associated with the class $l_{\im}$ of $\im$.  Similarly to
the max-pooling function we express this by introducing a sequence of
size $\lceil \log_2c \rceil$ of binary variables; the sequence's
binary number $b$ denotes the node from the output layer that is
associated with class $b+1 \in [c]$ and whose linear activation is
larger than the linear activation of the node associated with
$l_{\im}$.
\begin{align*}
	&\ws_{l_{\im}}(\vval[]{n-1}) \leq \ws_j(\vval[]{n-1})    +  \\
	& \hskip 1em 	M\sum_{k \in
	|\mathbf{c}|} \left( \mathbf{j_k} + (1 - 2 \mathbf{j_k})
	\pdelta[k]{n}\right), \\
        &  \hskip 1em j
	\in [0,c-1] \setminus \set{l_{\im}-1}  \tag{$C_{19}$}
\end{align*}
Again, we prevent $\pdelta[0]{n},\ldots,\pdelta[\lceil \log_2 c \rceil]{n}$ from
representing   $l_{\im}-1$ or any number greater than $c-1$:
\begin{align*}
	&\sum_{k \in |\mathbf{c}|} \mathbf{j_k} + (1 - 2 \mathbf{j_k})
	\pdelta[k]{n} \geq 1, \\
        	& \hskip 1em j \in \set{l_{\im}-1} \cup [c,2^{|\mathbf{c}|}-1]
	\tag{$C_{20}$}.
\end{align*}

The linear encoding of local robustness is then defined by $\con{lrob} =
C_{19} \cup C_{20}$. 

In summary, we have shown that all the affine and photometric
transformations considered are naturally captured by appropriate sets
of linear constraints. Moreover, we have shown that the property of
local transformational robustness can naturally be encoded into a set
of linear constraints. We now give the main result of this section by
formally linking local transformational robustness of a given CNN with
respect to a particular image and a transformation and the lack of a
solution for the corresponding MILP problem.

\begin{theorem}
	Let $\cnn$ be a CNN, $\mathfrak T$ a transformation with
	domain $\mathcal D$, and $\im$ an image.  Let $LP$ be the
	linear problem defined on objective function $obj=0$ and set
	of constraints
	$\con{\im} \cup \con{\overline{\cnn}} \cup \con{lrob}$, where
	$\con{\im} \triangleq \vval{0} = \im$ fixes the input of
	$\mathfrak T$ to $\im$.  Then $\cnn$ is locally
	transformationally robust for $\mathfrak{T}$ and $\im$ iff
	$LP$ has no solution.
\end{theorem}
\begin{proof}
	Let $\overline{\cnn} = \f{n+1}(\ldots \f{1}(\im)\ldots)$.  
	
	For the left to right direction assume that $LP$ has a
	feasible solution. Consider $d = (\lambda_{j} \colon
	j \in \dof(\mathfrak T))$, where each $\lambda_{j}$ is the
	value of the d.o.f. $d_j$ of $\mathfrak{T}$ in the solution.
	By the definition of $LP$ we have that $\cnn(\mathfrak T[d](\im))
	= \argmax(\vval[]{n+1})$. By the definition of $\con{lrob}$
	there is $l$ with $l \neq l_{\im}$ and
	$\vval[l]{n+1} \geq \vval[l_{\im}]{n+1}$. Therefore
	$\cnn( \mathfrak T[d](\im)) \neq l_{\im}$, and
	therefore $\cnn$ is not locally transformationally robust.

	For the right to left direction suppose that $\cnn$ is not
	locally transformationally robust. It follows that there is $d \in \mathcal D$
	such that $\cnn( \mathfrak T[d](\im)) \neq l_{\im}$. Then,  the assignment
	$\lambda_{j} = d_j$, for each d.o.f $j$ of $\mathfrak T$ and
	$\vval[j]{i} = \left( \f{i}(\ldots\f{1}(\im)\ldots)\right)_j$ is a
	feasible solution for $LP$.
\end{proof}

\begin{table*}[!t]
\centering
\begin{scriptsize}
\begin{tabular}{|c|c|c||c|c|c||c|c|c||c|c|c|}
	\hline
	\multicolumn{3}{|c|}{{\bf Translation}} &   
	\multicolumn{3}{|c|}{{\bf Subsampling}} &  
	\multicolumn{3}{|c|}{{\bf Zooming}}  &
	\multicolumn{3}{|c|}{{\bf Photometric}} 
	
	\\
		\hline
		\hline	
		$\mathcal D$ & 
			$\#\checkmark$ (s) & 
			LTR &  
		$\mathcal D$ 
			& $\#\checkmark$ (s) 
			& LTR &
		$\mathcal D$ 
			& $\#\checkmark$ (s) 
			& LTR &
		$\mathcal D$ 
			& $\#\checkmark$ (s) 
			& LTR 
	\\
	\hline
	\hline
		$[-1,1]^2$
			& $5 (48)$  
			& $0$
		& $[2,3]$ 
			& $97 (2)$ 
			& $0$
		& $[2,3]$
			& $94 (3)$
			& $4$ 
		& $[0.01,0.03] \times [0.99,1.01]$ 
			& $100 (1.1)$ 
			& $2$
	\\
	\hline
		$-3,3]^2$
			& $63 (15)$  
			& 0 &  
		$[2,5]$ 
			& $96 (3)$ 
			&0 &
		$[2,5]$
			& $90 (1)$
			& 7 &
			$[0.01,0.05] \times [0.97,1.03]$ 
			&  $100 (8.23)$ 
			& 2 
	\\
	\hline
		$[-5,5]^2$
			& $87 (26)$  
			& 0 &  
		$[2,7]$ 
			& $95 (2)$ 
			& 0 &
		$[2,7]$ 
			& $97 (4)$
			& 2 &
		$[0.01,0.07] \times [0.9,1.1]$ 
			& $100 (8.36)$ 
			& 2 
	\\
	\hline
		$[-7,7]^2$
			& $90 (22)$  
			& 0 &  
		$[2,9]$ 
			&$96 (3)$ 
			& 7 
		& $[2,9]$ 
			& $96 (3)$ 
			& 7 &
		$[0.1,0.5] \times [0.7,1.3]$
			& $100 (10.05)$ 
			& 2 
	\\
	\hline
		$ [-9,9]^2 $ 
			& $90 (20)$  
			& 0 &  
		$[2,11]$ 
			& $93 (1)$ 
			& 0 &
		$[2,11]$ 
			& $96 (3)$ 
			& 7 &
		$[0.1,0.7] \times [0.5,1.5]$
			& $97 (5.51)$ 
			& 0 
\\
\hline
\end{tabular}

\end{scriptsize}

\caption{Experimental results.}
\label{table:exp}
\end{table*}

Following the above the local transformational robustness of CNNs can
be established whenever the corresponding MILP problem has no
solution.  Otherwise, if a feasible solution for the problem can be
found, then the solution corresponds to the instantiations of the
transformations under analysis that violate the robustness of the
CNN. From these, it is possible to determine the sets of images 
responsible for the violation. These exemplars are
often called adversarial examples in related literature.  Finally note
that the result holds for the composition of any finite number of
transformations considered in the paper.

We conclude by noting that finding adversarial examples can be used to
improve the accuracy of the classifier, e.g., via retraining. On the other
hand showing
local transformational robustness can be used to build a formal safety
case for the perception system in question.

\section{Implementation and experimental results}\label{sec:experiments}

We have implemented the procedures from the previous section as a
toolkit \texttt{VENUS}~\cite{Venus18}. The toolkit takes as input the
descriptions of a CNN and a sequence of transformations against which
the local transformational robustness of the CNN is meant to be
assessed. Following this \texttt{VENUS} builds the linear encoding of
the verification query layer by layer and following all the
transformations in the input by adding the constraints discussed in
the preceding Section.

Having constructed the linear program the verifier invokes the Gurobi
checker~\cite{Gurobi+16a} to ascertain whether the program admits a
solution. The {\em satisfiability} output of the latter corresponds to
the violation of the local transformational robustness property of the
CNN, whereas the {\em unsatisfiability} output can be used to assert
the CNN is transformationally locally robust.

We have tested \texttt{VENUS} on networks trained on the MNIST
dataset~\cite{Lecun+98} using the deep learning toolkit
\texttt{Keras}~\cite{Chollet15}.
To the best of our knowledge currently there are no other methods or
tools for the same problem; so, in what below we report only the
results obtained with \texttt{VENUS}.
In  the experiments we fixed a CNN of 1481 nodes, with a convolutional layer of~3
convolutions with kernels of size~$15 \times 15$ and pool-size~$2
\times 2$, and an output layer with~$10$ nodes. The accuracy of the
network on the MNIST dataset is~93\%.  To check the network's
transformational local robustness we selected~100 images for which the
network outputs the correct classification label. We then performed
experiments for all of the transformations with varying domains for
each of their degrees of freedom.  The experiments were run on a
machine equipped with an i7-7700 processor and 16GB of RAM and running
Linux kernel version~4.15.0.  Table~\ref{table:exp} summarises the
results.

For every transformation and its domains considered, we report the
number of images that have been verified (irrespective of whether
these were shown robust) under the timeout of~200s followed by the
average time taken for verifying said images. This is indicated in the
column \#$\checkmark$ (s). For example, for subsampling
transformations on domain $[2,3]$, the method could verify 97 images
out of 100 with an average time of 2 secs. Following this, in the LTR
column we report the number of these that were assessed locally
transformationally robust.

Note that there is some variability in the results. For example,
several images could not be assessed by the timeout for the
translation with domain $[-1,1]$ but many more could be analysed under
the translation domain $[-3,3]$. This is likely to be due to
optimisations carried out by Gurobi which can be applied in some cases
only. Indeed, note in general that an increase in the range of the
domains does not lead to longer computations times, since the
resulting linear program is only marginally extended.

In summary, the results show that the CNN built from the MNIST dataset
is not locally transformationally robust with respect to translation,
subsampling and zooming, returning different classifications even for
small transformational changes to the input. The CNN appears just as fragile in terms of
luminosity and contrast changes.
Overall, the
results show that the CNN in question is brittle with respect to
transformational robustness. We stress that the aim of this work is
not to build robust CNNs, but to provide a provably correct and
automatic method that can be used to make this assessment
systematically. If the CNN that we analysed was intended to be used in
as part of a perception module in a safety-critical autonomous system,
our analysis would suggest that the CNN may be inadequate.

Still, in the spirit of adversarial retraining, the technique here
described can be used to augment the training set so as to improve the
robustness of the network under analysis~\cite{Goodfellow+18}. To test
this we trained a CNN with images from the MNIST dataset (60000
images) and recorded its classification accuracy as 19\% on
transformed MNIST images. We
then used \texttt{VENUS} to generate over~10000 transformed images
(for compositions of transformations here considered)
which we appended to the training set. The augmented set was then used
to retrain the network resulting in a classification accuracy of~26\%.
This exemplifies \texttt{VENUS} not only as a verification tool but
also as a tool to improve the robustness of CNNs.


We conclude by noting that we conducted further tests with images of
different sizes. As expected, we found that the effectiveness of the
toolkit is hindered by increases in the size of the network, as these
lead to heavier MILP encodings. This has already been reported in the
literature~\cite{Dvijotham+18} and points to further work to improve
scalability further.

\section{Conclusions}\label{sec:conclusions}
Forthcoming autonomous systems are expected to rel. 
aspects of modern society. Several of these, including autonomous
vehicles, are effectively safety-critical systems, i.e., they are
systems in which malfunctions may lead to loss of life. Differently
for other non-autonomous safety-critical systems, such as avionics
systems, no methods presently exist to give formal guarantees to large
classes of autonomous systems. Providing formal guarantees for autonomous
systems in which learning-enabled components are present (i.e.,
components synthesised from data via machine learning methods) is
presently an open challenge.  We believe that research is urgently
needed to address this deficiency.

To this end, in this paper we have considered a class of perception
systems often present in autonomous systems, i.e., CNN-based
classifiers commonly used in the perception modules of autonomous
systems. While CNNs are regularly used in various applications where
their correctness and accuracy is not essential, their use in
autonomous systems requires a high degree of assurance. Typically this
is either directly provided by, or aided by formal methods, as in the
case of DO-333/ED-216 in avionics. This is also likely to be required
for certification purposes of the system.

To this end, we have introduced a notion of robustness, called local
transformational robustness, which can be used to give guarantees on
the perception system under analysis. Local transformational
robustness encodes the ability of a classifier to withstand variations
of the image in question with respect to its classification. These
modifications, encoded mathematically as various kinds of
transformations, are meant to capture those induced by the normal
variability of devices and environment conditions at runtime. We
developed a method to assess local transformational robustness
automatically by translating this to MILP problems. The experimental
results showed that the toolkit we introduced
returns correct and viable results against the MNIST dataset for 
all the transformations considered.

In further work, firstly, we plan to improve the performance of the
back-end for the problems considered including experimenting with other
technologies such as SMT and mixed approaches. Secondly, while we focused
on local properties in this paper, i.e., properties concerning a single
image, we intend to extend the work to study (global) robustness  with
respect to classes of images. Lastly, while the paper focused on images,
the notion of local transformational robustness is general to be applied in
other contexts where CNNs are used including audio recognition.

 {\bf Acknowledgements. }  The authors would like to thank Daniel
 Rueckert and Rahul Sukthankar for comments on a previous version of
 this article. This work is partly funded by DARPA under the Assured
 Autonomy programme.

{\small
\bibliographystyle{ieee}
\bibliography{../../../bib.bib}
}

\end{document}